\documentclass{article}

\usepackage{arxiv}
\usepackage{url}
\usepackage{amsthm}
\usepackage[utf8]{inputenc}
\usepackage{textcomp}
\usepackage{fix-cm}
\usepackage{xypic}  
\usepackage{subfig} 

\usepackage{xcolor}  

\usepackage{amssymb}

\usepackage{amsmath}
\usepackage{mathdots}

\usepackage{mathtools} 

\usepackage{tensor}

\usepackage{urwchancal}
\DeclareFontFamily{OT1}{pzc}{}
\DeclareFontShape{OT1}{pzc}{m}{it}{<-> s * [1.10] pzcmi7t}{}
\DeclareMathAlphabet{\mathpzc}{OT1}{pzc}{m}{it}

\usepackage{graphicx}
\usepackage{ifpdf}
\ifpdf
\usepackage{epstopdf}
\PrependGraphicsExtensions{.svg}
\fi

\newtheorem{theorem}{Theorem}[section]
\newtheorem{lemma}[theorem]{Lemma}

\providecommand{\N}{\mathbb{N}}
\providecommand{\R}{\mathbb{R}}

\providecommand{\SO}{\mathbf{SO}}

\providecommand{\SE}{\mathbf{SE}}
\providecommand{\SOT}{\mathbf{SOT}}
\providecommand{\SIM}{\mathbf{SIM}}

\providecommand{\VSLAM}{\mathbf{VSLAM}}

\providecommand{\so}{\mathfrak{so}}
\providecommand{\se}{\mathfrak{se}}
\providecommand{\sot}{\mathfrak{sot}}

\providecommand{\vslam}{\mathfrak{vslam}}

\providecommand{\Sph}{\mathrm{S}}
\providecommand{\RP}{\R\mathbb{P}}

\providecommand{\calM}{\mathcal{M}}
\providecommand{\calN}{\mathcal{N}}

\providecommand{\totT}{\mathcal{T}}

\providecommand{\vecV}{\mathbb{V}}

\providecommand{\id}{\mathrm{id}} 

\providecommand{\Lyap}{\mathcal{L}} 

\providecommand{\td}{\mathrm{d}}
\providecommand{\tD}{\mathrm{D}}

\providecommand{\mr}[1]{{#1}^\circ} 

\providecommand{\ob}[1]{\overline{#1}} 
\usepackage{accents}
\makeatletter
\providecommand{\scirc}{%
    \hbox{\fontfamily{\rmdefault}\fontsize{0.4\dimexpr(\f@size pt)}{0}\selectfont{\raisebox{-0.52ex}[0ex][-0.52ex]{$\circ$}}}}
\DeclareRobustCommand{\mathcirc}{\accentset{\scirc}}
\makeatother
\providecommand{\obb}[1]{\mathrlap{\overline{#1}}\mathcirc{#1}}

\mathchardef\mhyphen="2D

\providecommand{\etal}{\textit{et al.}~}

\renewcommand{\mr}[1]{#1^\circ}

\usepackage{hyperref}

\begin{document}

\newcommand{\publicationdetails}
{
  \copyrightNoticeIEEE{2019}
  P. v. Goor, R. Mahony, T. Hamel and J. Trumpf, "A Geometric Observer Design for Visual Localisation and Mapping," \textit{2019 IEEE 58th Conference on Decision and Control (CDC)}, Nice, France, 2019, pp. 2543-2549,
  \DOILink{https://ieeexplore-ieee-org.virtual.anu.edu.au/abstract/document/9029435}{10.1109/CDC40024.2019.9029435}
}
\newcommand{\publicationversion}
{Author accepted version}

\title{A Geometric Observer Design for Visual Localisation and Mapping}
\headertitle{A Geometric Observer Design for Visual Localisation and Mapping}

\author{
\href{https://orcid.org/0000-0003-4391-7014}{\includegraphics[scale=0.06]{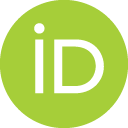}\hspace{1mm}
Pieter van Goor} \\
Department of Electrical, Energy and Materials Engineering\\
Australian National University\\
  ACT, 2601, Australia \\
\texttt{Pieter.vanGoor@anu.edu.au}
\\ \And 
\href{https://orcid.org/0000-0002-7803-2868}{\includegraphics[scale=0.06]{orcid.png}\hspace{1mm}
Robert Mahony} \\
Department of Electrical, Energy and Materials Engineering\\
Australian National University\\
  ACT, 2601, Australia \\
\texttt{Robert.Mahony@anu.edu.au}
\\ \And	
\href{https://orcid.org/0000-0002-7779-1264}{\includegraphics[scale=0.06]{orcid.png}\hspace{1mm}
Tarek Hamel} \\
I3S (University C\^ote d'Azur, CNRS, Sophia Antipolis)\\
and Insitut Universitaire de France\\
\texttt{THamel@i3s.unice.fr}
\\ \And	
\href{https://orcid.org/0000-0002-5881-1063}{\includegraphics[scale=0.06]{orcid.png}\hspace{1mm}
Jochen Trumpf} \\
Department of Electrical, Energy and Materials Engineering\\
Australian National University\\
  ACT, 2601, Australia \\
\texttt{Jochen.Trumpf@anu.edu.au}
}


%
%

\maketitle

\begin{abstract}
This paper builds on recent work on Simultaneous Localisation and Mapping (SLAM) in the non-linear observer community, by framing the visual localisation and mapping problem as a continuous-time equivariant observer design problem on the symmetry group of a kinematic system.
The state-space is a quotient of the robot pose expressed on $\SE(3)$ and multiple copies of real projective space, used to represent both points in space and bearings in a single unified framework.
An observer with decoupled Riccati-gains for each landmark is derived and we show that its error system is almost globally asymptotically stable and exponentially stable in-the-large.
\end{abstract}



\section{Introduction}\label{sec:intro}

Simultaneous Localisation and Mapping (SLAM) is a well-known problem in mobile robotics and has been an active area of research for the last 30 years \cite{2015_Manuel_vslam}.
Visual localisation and mapping refers to the particular case of the SLAM problem where the only exteroceptive sensors available are cameras.
The visual localisation  and mapping problem, and particularly the case where only a single monocular camera is available, continues to be of substantial interest due to the low cost and low weight, as well as the ubiquity of single camera systems \cite{2015_Manuel_vslam}.
While visual localisation and mapping is an established research topic with a rich history \cite{2016_Cadena_TRO}, it remains an active research topic, especially in the area of low-cost light-weight embedded systems
\cite{2018_Delmerico_icra}.
State-of-the-art filters and observers approach the SLAM problem through linearisation, and do not deal well with poor initial estimation or choice of linearisation point \cite{2016_Cadena_TRO}.
Additionally, these methods suffer from high computational complexity and poor scalability \cite{2015_Manuel_vslam, 2012_Strasdat_CVIU}.

Both the SLAM and visual localisation and mapping problems have attracted interest recently in the non-linear observer community.
Approaches to these problems have emerged from earlier work on attitude estimation \cite{2008_Mahony_tac,2008_Bonnabel_TAC} and pose estimation \cite{2009_Baldwin_icra,2010_Vasconcelos.SCL,RM_2011_Hua_cdc}.
Bonnabel \etal \cite{2016_Barrau_arxive} exploited a novel Lie group to design an invariant Kalman Filter for the SLAM problem.
Parallel work by Mahony \etal \cite{2017_Mahony_cdc} developed the same Lie group and proposed a quotient manifold structure for the state-space of the SLAM problem.
Work by Zlotnik \etal \cite{2018_forbes_TAC} derives a geometrically motivated observer for the SLAM problem that includes estimation of bias in linear and angular velocity inputs.
For the visual localisation  and mapping problem, where only bearing measurements are available, Lourenco \etal \cite{2016_LouGueBatOliSil,2018_Lourenco_RAS} proposed an observer with a globally exponentially stable error system using depths of landmarks as separate components of the observer.
Grabe \etal \cite{2015_grabe_IJRR} derived a non-linear observer for the case where a significant number of the bearings measured are of coplanar landmarks by using the instantaneous homography constraint.
Bjorne \etal \cite{2017_bjorne_fusion} uses an attitude heading reference system (AHRS) to determine the orientation of the robot, and then solves the SLAM problem using a linear Kalman filter.
A similar approach to the visual localisation and mapping case is undertaken in \cite{2017_LebHamMahSam}.
Hamel \etal have also introduced a Riccati observer \cite{2018_Hamel_TAC} for the case where the orientation of the robot is known.

In this paper we present a novel non-linear geometric observer for the visual localisation and mapping problem.
The approach extends the SLAM manifold presented in \cite{2017_Mahony_cdc} to include bearings (such as magnetometer or gravity measurements) and landmark points in the same formulation by exploiting the structure of the real-projective space $\RP^3$ and homogeneous coordinates for bound and free vectors.
The proposed $\RP^3$ state-space also allows modelling of visual features as a simple linear projection of $\RP^3$ onto $\RP^2$.
A novel Lie group termed the $\VSLAM_n(3)$ group is introduced and shown to be a symmetry on the measurement function of the visual localisation and mapping problem.
The proposed observer uses decoupled gain matrices for each landmark point that satisfy a simple Riccati equation.
As a consequence of decoupling the Riccati observer for each landmark, the computational complexity of our approach is only $\mathcal{O}(n)$.
Finally, the innovation on the pose of the robot is determined through finding the minimum of a novel cost function on the tangent space of $\RP^3$, and is based on the static environment assumption common in SLAM algorithms.
The resulting observer is shown to have an error system that is almost globally asymptotically stable (the basin of attraction excludes a set of measure zero) and exponentially stable in-the-large (exponentially stable on any compact set contained in the basin of attraction).

This paper consists of five sections alongside the introduction and conclusion.
Section \ref{sec:preliminaries} introduces key notation and identities, and provides an in-depth explanation of the application of $\RP^3$ to representing points and bearings in 3d space.
In Section \ref{sec:problem-formulation}, we formulate the kinematics, state-space and output of the visual localisation  and mapping system, and in Section \ref{sec:symmetry} we introduce the new Lie group $\VSLAM_n(3)$ that acts on the state-space.
In Section \ref{sec:observer} we derive a non-linear observer on the Lie group, and in Section \ref{sec:simulation} we provide the results of a simulation.
The experimental results are designed to verify the theory developed throughout the paper, not to provide a comprehensive evaluation of performance.

\section{Preliminaries} \label{sec:preliminaries}
\subsection{Notation}
The special orthogonal group and special Euclidean group are denoted $\SO(3)$ and $\SE(3)$ respecively, with Lie algebras $\so(3)$ and $\se(3)$.
For any $\Omega = (\Omega_1, \Omega_2, \Omega_3) \in \R^3$, the corresponding skew-symmetric matrix is denoted by
\begin{align*}
\Omega^\times := \left( \begin{matrix}
0 & -\Omega_3 & \Omega_2 \\
\Omega_3 & 0 & -\Omega_1 \\
-\Omega_2 & \Omega_1 & 0
\end{matrix} \right) \in \so(3).
\end{align*}
This matrix has the property that, for any $v \in \R^3$, $\Omega^\times v = \Omega \times v$ where $\Omega \times v$ is the vector (cross) product between $\Omega$ and $v$.

Consider a matrix $P \in \SE(3)$.
The notations $R_P \in \SO(3)$ and $x_P \in \R^3$ are used to represent the rotation and translation components of $P$ respectively, and $P$ may be written as
\begin{align*}
P = \begin{pmatrix}
R_P & x_P \\ 0 & 1
\end{pmatrix}.
\end{align*}
Likewise, for a matrix $U \in \se(3)$, the notations $\Omega_U \in \so(3)$ and $V_U \in \R^3$ represent the rotational and translational velocity components of $U$ respectively, and $U$ may be written as
\begin{align*}
U = \begin{pmatrix}
\Omega_U & V_U \\ 0 & 1
\end{pmatrix}.
\end{align*}

For any $y \in \R^3\setminus \{0\}$ the \textit{projector} $\Pi_y$ is given by
\begin{align*}
\Pi_y := I_3 - \frac{y y^\top}{|y|^2}.
\end{align*}
The operator $\Pi_y$ projects vectors onto the subspace of $\R^3$ orthogonal to $y$.
The projector and the skew-symmetric matrix are related by
\begin{align} \label{eq:projector-skew-identity}
\Pi_y = - \frac{y^\times y^\times }{|y|^2},
\end{align}
for any $y \in \R^3 \setminus \{0 \}$.
For any $\bar{y} \in \R^4\setminus \{0\}$ the projector is similarly defined as
\begin{align*}
\ob{\Pi}_{\bar{y}} := I_4 - \frac{\bar{y} \bar{y}^\top}{|\bar{y}|^2}.
\end{align*}

\subsection{Real Projective Space}
For $x \in \R^4 \setminus \{0\}$, define the set of equivalence classes
\begin{align*}
[x] := \left\{ a x \ \vline \ a \in \R \setminus \{0\} \right\}.
\end{align*}
Given two elements $x,y \in \R^4 \setminus \{0\}$, the notation $x \simeq y$ indicates $x = ay$ for some $a \in \R \setminus \{0\}$.
The 3-dimensional real-projective space $\RP^3 = \{[x] \ \vline\  x\in \R^4 \setminus\{0\}\}$ is a smooth quotient manifold \cite{RM_2008_Absil.etal}.
For any full rank matrix $A \in \R^{4 \times 4}$, the operation
\begin{align}
A[x] := [Ax]
\label{eq:Ax_RP3}
\end{align}
is well-defined.

Let $x \in \R^4 \setminus \{ 0 \}$, and define an horizontal space $H_x = \{ v \in \R^4 \; \vline \; v^\top x = 0\}$.
Define an equivalence relationship $(x, v) \equiv (ax , a v)$ for $a \in \R \setminus \{0\}$ between elements of $H_x$ and $H_{ax}$.
A tangent vector $v_{[x]} \in T_{[x]} \RP^3$ is the equivalence class
$[x, v] = \{ (ax,av) \; \vline \; v \in H_x \}$. 

For any $[x] \in \RP^3$, define the projector
\[
\ob{\Pi}_{[x]} := \ob{\Pi}_{x}.
\]
To see this is well-defined, let $a \in \R$ be a non-zero scalar, and check
\begin{align*}
\ob{\Pi}_{ax} = I_4 - \frac{(ax) (ax)^\top}{\vert (ax) \vert^2} = I_4 - \frac{a^2}{a^2} \frac{x x^\top}{\vert x \vert^2} = \ob{\Pi}_x.
\end{align*}
Analogously, the projector $\Pi_{[y]} := \Pi_y$ is well-defined for any $y \in \RP^2$.

Let $p \in \R^3$ be a vector representing the position of a point in space.
Define the homogeneous coordinates
\[
\ob{p} := \left( \begin{matrix} p \\ 1 \end{matrix} \right)
\]
as an embedding $\R^3 \hookrightarrow \R^4$ and refer to such points $\ob{p}$ as \textit{bound vectors} with foot at the origin of the reference frame and tip at the $\R^3$ point it represents.
Let $b \in \Sph^2 = \{  b \in \R^3 \; |\; |b| = 1\}$ be a vector representing a bearing or direction and define homogeneous coordinates
\[
\obb{b} = \left( \begin{matrix} b \\ 0 \end{matrix} \right)
\]
as an embedding $\Sph^2 \hookrightarrow \R^4$.  We term $\obb{b}$ a \textit{free vector}.
Using these embeddings it is possible to define a map $\alpha : \R^3 \sqcup \Sph^2 \rightarrow \RP^3$
\begin{align*}
\alpha(p) & :=  [\ob{p} ],  \quad\quad p \in \R^3,  \\
\alpha(b) & := [\obb{b} ], \quad\quad  b \in \Sph^2.
\end{align*}
A \textit{point-type} element of $\RP^3$ is any element in the subset $\{ [x] \ | \ x_4 \neq 0 \}$.
A \textit{bearing-type} element of $\RP^3$ is any element in the subset $\{ [x] \ | \ x_4 = 0 \}$.
A full inverse of $\alpha$ is not uniquely defined due to the sign ambiguity of elements of $\RP^3$.
However, it is possible to define a unique map $\gamma : \RP^3 \rightarrow \R^3 \sqcup \RP^2$ by
\begin{align} \label{eq:inverse-rp3-map-gamma}
\gamma([x]) & := \left\{
\begin{array}{ll}
x_{1:3}{/} x_4 \in \R^3, & \text{ if } x_4 \not= 0 \\
\left[ x_{1:3} \right] \in \RP^2, & \text{ if } x_4 = 0
\end{array}
\right. ,
\end{align}
where $x_{1:3} \in \R^3$ denotes the first three elements of $x$ and $[x_{1:3}] = \{ a x_{1:3} \; |\; a \in \R \setminus \{0\} \}$, analogous to the $\R^4$ definition.
Define a projection $\beta : \R^3 \sqcup \Sph^2 \to \R^3 \sqcup \RP^2$ by
\begin{align*}
\beta (x)  :=
\left\{ \begin{array}{ll}
x \in \R^3, & \text{ if } x \in \R^3 \\
\left[ x \right] \in \RP^2 & \text{ if } x \in \Sph^2
\end{array}
\right. .
\end{align*}
The following commutative diagram holds
\[
\xymatrix{
\R^3 \sqcup \Sph^2   \ar@{->}[rd]^{\alpha}  \ar@{->}[d]_{\beta}&
 \\
\R^3 \sqcup \RP^2  & \RP^3 \ar@{->}[l]^{\quad \gamma}
}
\]
The map $\gamma$ is smooth under restriction to either point-type elements or bearing-type elements of $\RP^3$.
Although $\gamma$ is unable to reconstruct the full direction vector $b$ from a bearing-type $\RP^3$ element, the unsigned direction $[b]$ is sufficient for the observer construction that we undertake in the sequel.

\section{Problem Formulation} \label{sec:problem-formulation}

\subsection{VSLAM Total Space}
The formulation of the total space for the VSLAM problem is an extension of the formulation in \cite{2017_Mahony_cdc} to include not only points in 3D space but also bearings through their $\RP^3$ representations.

\textit{Raw coordinates} for the VSLAM problem can be defined by fixing an arbitrary reference frame $\{0\}$.
Let $P \in \SE(3)$ and $\eta_i \in \RP^3$ represent the robot pose and landmark coordinates respectively, defined with respect to $\{0\}$.
Note that each $\eta_i \in \RP^3$ is either point-type or bearing-type depending on whether its last entry is zero.
The total space of the VSLAM problem is the product space
\begin{align*}
\totT_n(3) = \SE(3) \times \RP^3 \times \cdots \times \RP^3,
\end{align*}
with elements
\begin{align*}
(P, \eta_1,..., \eta_n).
\end{align*}
The notation $(P, \eta_i) \equiv (P, \eta_1,..., \eta_n)$ is used to simplify notation in the sequel.

Given $(P, \eta_i) \in \totT_n(3)$, recalling \eqref{eq:Ax_RP3} define
\begin{align*}
\lfloor P, \eta_i \rfloor  := \left\{ (S^{-1}P, S^{-1} \eta_i) \ | \ S \in \SE(3) \right\}.
\end{align*}
Given two elements $(P, \eta_i), (Q, \theta_i) \in \totT_n(3)$, the notation $(P, \eta_i) \simeq (Q, \theta_i)$ means that $(P, \eta_i) = (S^{-1} Q, S^{-1} \theta_i)$ for some $S \in \SE(3)$.
The SLAM manifold is the set
\begin{align*}
\calM_n(3) = \left\{ \lfloor P, \eta_i \rfloor \ | \ (P, \eta_i) \in \totT_n(3) \right\},
\end{align*}
with quotient manifold structure \cite{2017_Mahony_cdc}.

An expression is well-defined on the SLAM manifold $\calM_n(3)$ if it is invariant to the action of a rigid-body transformation of the reference frame.
An important example is $(P, \eta_i) \mapsto P^{-1} \eta_i$.
Given any $S \in \SE(3)$, one has
\begin{align} \label{eq:welldef-of-body-coordinates}
(S^{-1} P, S^{-1} \eta_i) \mapsto (S^{-1}P)^{-1} S^{-1} \eta_i = P^{-1} S S^{-1} \eta_i = P^{-1} \eta_i.
\end{align}

\subsection{VSLAM Kinematics}
The assumption will be made that the robot is moving through a static environment.
Consider the velocity input space $\vecV = \se(3)$.
The kinematics of the VSLAM system are given by the function
\begin{align} \label{eq:input_function_f}
f:& \totT_n(3) \times \vecV \to T\totT_n(3), \notag \\
&((P,\eta_i) , U) \mapsto (PU,0).
\end{align}

\subsection{System Output}
The physical measurements taken by our robot in the VSLAM system are the bearings of landmarks.
Let $\eta'_i = P^{-1} \eta_i$ be the body-fixed frame coordinates of a landmark $\eta_i \neq \alpha(x_P)$.
Using the basic pinhole camera model as described in \cite{2003_hartley_multiview} with invertible $3\times 3$ camera matrix $K$, the measurement of $\eta'_i$ taken by the camera is $\begin{pmatrix} K & \mathbf{0}_{3 \times 1} \end{pmatrix} \eta'_i$.
Assuming the camera is calibrated matrix $K$, it is easy to recover the element
\begin{align*}
K^{-1} \begin{pmatrix} K & \mathbf{0}_{3 \times 1} \end{pmatrix} \eta'_i &= \begin{pmatrix} I_3 & \mathbf{0}_{3 \times 1} \end{pmatrix} \eta'_i
\end{align*}
although the scale of this element is arbitrary and cannot be known.
If $\eta_i$ is a bearing-type element, then $\theta_{i,4} = 0$ and no information is lost through the camera projection.
However, if $\eta_i$ is a point-type element, then the scale of the vector is not recoverable.
In this formulation the sign of the landmark measurement (representing whether the landmark is in front of or
behind the camera) is ambiguous, but this is sufficient for the observer design undertaken in Section \ref{sec:observer}.
The choice of bearing-type or point-type for a particular landmark $\eta_i$ is a modelling choice based on the requirements for the resulting map of the environment.

The output space of the VSLAM system is defined as
\begin{align*}
\calN_n(3) := \RP^2 \times \cdots \times \RP^2.
\end{align*}
The output function of the VSLAM system is defined as
\begin{align} \label{eq:output_function_h}
h&:\totT_n(3) \to \calN_n(3), \notag \\
&\phantom{:} (P, \eta_i) \mapsto \begin{pmatrix}
I_3 & \bf{0}
\end{pmatrix} P^{-1} \eta_i.
\end{align}
The output function transforms each $\eta_i$ into body-fixed frame coordinates, and projects the result into $\RP^2$, representing bearing-type of point-type landmark measurements with a calibrated pinhole camera.

\section{Symmetry of the VSLAM Problem} \label{sec:symmetry}
\subsection{Symmetry of the Total Space}

We introduce a group we term Scaled Orthogonal Transformations $\SOT(n)$, a subgroup of the group of similarity transforms on $\R^n$.

\begin{lemma} \label{def:SOT_group}
For any $n \in \N$, the set
\begin{align*}
\SOT(n) = \left\{ \left( \begin{matrix}
R & 0 \\ 0 & a
\end{matrix} \right) \ \vline \ R \in \SO(n), a \in \R \setminus \{ 0 \}  \right\},
\end{align*}
with matrix multiplication is a subgroup of $\SIM(n)$.
\end{lemma}
\begin{proof}
Assigning matrix multiplication as the group action it is clear that $\SOT(n)$ is the direct product of $\SO(3) \times \R_*$, where $R_*$ is the Lie group formed by assigning multiplication as the operation on $\R \setminus \{ 0 \}$.
It is straightforward to verify that $\SOT(n)$ is a subgroup of $\SIM(n)$ by considering the action $x\mapsto \frac{1}{a} R x$ for $x\in \R^n$.
\end{proof}

The action of $\SOT(3)$ on landmarks is a rotation combined with a scaling for point-type landmarks. Recalling \eqref{eq:Ax_RP3} and taking advantage of the equivalence class structure of $\RP^3$,
\begin{align*}
\begin{pmatrix}
R & 0 \\ 0 & a
\end{pmatrix}
\begin{bmatrix} p \\ 1 \end{bmatrix}
=
\begin{bmatrix} R p \\ a \end{bmatrix}
=
\begin{bmatrix} \frac{1}{a} R p \\ 1 \end{bmatrix}, \hspace{1cm}
\begin{pmatrix}
R & 0 \\ 0 & a
\end{pmatrix}
\begin{bmatrix} b \\ 0 \end{bmatrix}
=
\begin{bmatrix} R b \\ 0 \end{bmatrix} .
\end{align*}
There are exactly three orbits of $\SOT(3)$ acting on $\RP^3$, defined by
\begin{align} \label{eq:RP3-orbits}
\RP^3_p &:= \left\{ [x] \in \RP^3 \ \vert \ x_4 \neq 0, \ [x] \neq  \alpha(\bf{0}) \right\}, \notag \\
\RP^3_b &:= \left\{ [x] \in \RP^3 \ \vert \ x_4 = 0 \right\}, \notag \\
\RP^3_0 &:= \left\{  \alpha(\bf{0}) = \bf{e}_4 \right\},
\end{align}
where $x_4$ refers to the fourth coordinate of $x$.

The symmetry group $\VSLAM_n(3)$ for the VSLAM problem with $n$ landmarks in 3 dimensions is defined as a Lie group
\begin{align*}
\VSLAM_n(3) = SE(3) \times \SOT(3) \times \cdots \times \SOT(3),
\end{align*}
with product Lie group structure.
The associated Lie algebra is denoted $\vslam_n(3)$.

\begin{lemma}
The mapping $\Upsilon : \VSLAM_n(3) \times \totT_n(3) \to \totT_n(3)$ defined by
\begin{align} \label{eq:group_action_upsilon}
\Upsilon((A,Q_i), (P, \eta_i)) = (PA, PAQ_i^{-1}P^{-1}\eta_i),
\end{align}
where the right-hand expression depends on definition \eqref{eq:Ax_RP3},
is a right group action of $\VSLAM_n(3)$ on $\totT_n(3)$.
\end{lemma}
\begin{proof}
Trivially, $\Upsilon((I_4, I_4), (P, \eta_i)) = (P, \eta_i)$ for any $(P, \eta_i) \in \totT_n(3)$.
Let $(A_1,Q_{i,1}), (A_2,Q_{i,2}) \in \VSLAM_n(3)$ and $(P, \eta_i)$ be arbitrary. Then
\begin{align*}
\Upsilon(&(A_1, Q_{i,1}), \Upsilon((A_2,Q_{i,2}), (P,\eta_i))) \\
&= \Upsilon((A_1,Q_{i,1}), (PA_2, PA_2 Q_{i,2}^{-1}P^{-1}\eta_i)), \\
&= (PA_2A_1, PA_2A_1 Q_{i,1}^{-1} (PA_2)^{-1}PA_2 Q_{i,2}^{-1}P^{-1}\eta_i), \\
&= (P(A_2A_1), P(A_2A_1) (Q_{i,2} Q_{i,1})^{-1}P^{-1}\eta_i), \\
&= \Upsilon((A_2, Q_{i,2}) \cdot (A_1,Q_{i,1}), (P, \eta_i)).
\end{align*}
This demonstrates that $\Upsilon$ is a right action as required.
\end{proof}

Recall the orbits of $\SOT(3)$ described in \eqref{eq:RP3-orbits}.
Given a configuration $(\mr{P}, \mr{\eta}_i) \in \totT_n(3)$, let $(P, \eta_i) = \Upsilon((A,Q_i), (\mr{P}, \mr{\eta}_i))$ for some $(A, Q_i) \in \VSLAM_n(3)$.
Observe that if ${\mr{P}}^{-1} \mr{\eta}_j \in \RP^3_0$ for some $j$, then $P^{-1} \eta_j \in \RP^3_0$ also, independent of the particular element $(A, Q_i)$.
To overcome this, in the remainder of the paper it is assumed that there is never a $j$ such that ${\mr{P}}^{-1} \mr{\eta}_j \in \RP^3_0$.
This assumption is reasonable, in that it is equivalent to assuming there are no landmarks coinciding precisely with the origin of the robot.
Additionally, it is assumed that the type of each landmark (point or bearing) is known, and the landmarks are enumerated such that $i = 1,...,n_p$ and $i=n_p+1,...,n_p+n_b=n$ represent of point- and bearing-type landmarks respectively.
The reduced total space is defined as
\begin{align*}
\mr{\totT}_{n_p,n_j}(3) := \left\{ (P,\eta_i) \in \right. & \totT_{n_p+n_j}(3) \; \vline \; 1 \leq i \leq n_p \Leftrightarrow \eta_i \in \RP^3_p, \\
&\left.  1 \leq i-n_p \leq n_b \Leftrightarrow \eta_i \in \RP^3_p  \right\},
\end{align*}
and only elements $(P, \eta_i) \in \mr{\totT}_{n_p,n_j}(3)$ are considered from here going forward.

%
%

\subsection{Lift of the VSLAM Kinematics}
In order to consider the system on the $\VSLAM_n(3)$ group, the kinematics from the state space must be lifted onto the group.
The following lemma provides the lift function.

\begin{lemma}
The function $\lambda: \mr{\totT}_{n_p,n_j}(3) \times \vecV \to \vslam_n(3)$, defined by
\begin{align*}
\lambda((P, \eta_i), U) = (U, W(U, P^{-1}\eta_i)),
\end{align*}
where $W : \se(3) \times (\RP^3 \cup \RP^3_p) \to \sot(3)$ is given by
\begin{align*}
W\left( (\Omega_U, V_U), \begin{bmatrix}
q \\ r
\end{bmatrix} \right) = \left( \begin{matrix}
\left( \Omega_U - r \frac{V_U \times q}{|q|^2}\right)^\times & 0 \\ 0 & -r \frac{V_U^\top q}{|q|^2}
\end{matrix} \right),
\end{align*}
is a velocity lift of the kinematics \eqref{eq:input_function_f} onto $\VSLAM_n(3)$ with respect to the group action \eqref{eq:group_action_upsilon}.
\end{lemma}

\begin{proof}
To show that $\lambda$ is a velocity lift, it is required that
\begin{align*}
\tD \Upsilon_{(P, \eta_i)}(\id) \left[ \lambda((P, \eta_i), U) \right] = f((P, \eta_i), U).
\end{align*}
Equivalently, it is required to show that
\begin{align} \label{eq:lift_condition}
\left( PU, \ob{\Pi}_{\eta_i} \left(PUP^{-1} - PW_iP^{-1} \right) \eta_i \right) = (PU, 0),
\end{align}
where $W_i := W(U, P^{-1}\eta_i)$.

First, it is necessary to show that $W$ is well-defined whenever $q \neq \bf{0}$.
To see this, let $a \in \R$ be any non-zero scalar, and observe that
\begin{align} \label{eq:velocity-lift-scale-invariance}
&W\left( U, \begin{bmatrix}
aq \\ ar
\end{bmatrix} \right) =\left( \begin{matrix} \left( \Omega_U - ar \frac{V_U \times (aq)}{|aq|^2}\right)^\times & 0 \\ 0 & -ar \frac{V_U^\top (aq)}{|aq|^2}\end{matrix} \right), \notag \\
& \hspace{1cm} =\left( \begin{matrix} \left( \Omega_U - r \frac{V_U \times q}{|q|^2}\right)^\times & 0 \\ 0 & -r \frac{V_U^\top q}{|q|^2} \end{matrix} \right) = W\left( U, \begin{bmatrix} q \\ r \end{bmatrix} \right).
\end{align}

Recalling the expression for $f$ provided in \eqref{eq:input_function_f}, it is clear that the first terms on both sides of \eqref{eq:lift_condition} are equal.
Let
\begin{align} \label{eq:qr-def}
\left[ \begin{matrix}
q_i \\ r_i
\end{matrix} \right] := P^{-1} \eta_i.
\end{align}
In order to aid in the readability of the following equations, $q_i$ and $r_i$ in \eqref{eq:qr-def} are chosen such that $\vert q_i \vert = 1$.
However, it is important to note this choice is arbitrary as shown in \eqref{eq:velocity-lift-scale-invariance}.
To show \eqref{eq:lift_condition}, consider that
\begin{align*}
&\ob{\Pi}_{\eta_i} PUP^{-1}\eta_i = \ob{\Pi}_{\eta_i} PU \left[ \begin{matrix}
q_i \\ r_i
\end{matrix} \right], \\
&\hspace{0.5cm} = \ob{\Pi}_{\eta_i} P \left[ \begin{matrix}
\Omega_U^\times q_i + r_i V_U \\ 0
\end{matrix} \right], \\
&\hspace{0.5cm} = \ob{\Pi}_{\eta_i} \left[ P \left( \begin{matrix}
\Omega_U^\times q_i + r_i V_U \\ 0
\end{matrix} \right) - r_i V_U^\top q_i P\left( \begin{matrix}
q_i \\ r_i
\end{matrix} \right) \right], \\
&\hspace{0.5cm} = \ob{\Pi}_{\eta_i} P \left[ \begin{matrix}
\Omega_U^\times q_i + r_i \left( I_3 - q_i q_i^\top \right) V_U \\ -r_i V_U^\top q_i r_i
\end{matrix} \right].
\\
&\hspace{0.5cm} = \ob{\Pi}_{\eta_i} P \left[ \begin{matrix}
\Omega_U^\times q_i - r_i q_i^\times q_i^\times V_U \\ -r_i V_U^\top q_i r_i
\end{matrix} \right],
\end{align*}
using the identity \eqref{eq:projector-skew-identity}. This further reduces to
\begin{align*}
\ob{\Pi}_{\eta_i} PUP^{-1}\eta_i &= \ob{\Pi}_{\eta_i} P \left[ \begin{matrix}
\Omega_U^\times q_i - r_i (V_U^\times q_i)^\times q_i  \\ -r_i V_U^\top q_i r_i
\end{matrix} \right], \\
&= \ob{\Pi}_{\eta_i} P \left( \begin{matrix} \left( \Omega_U - r_i V_U \times q_i \right)^\times & 0 \\ 0 & -r_i V_U^\top q_i \end{matrix} \right) \left[ \begin{matrix}
q_i \\ r_i
\end{matrix} \right], \\
&= \ob{\Pi}_{\eta_i} PW_i P^{-1} \eta_i,
\end{align*}
where the last step follows from \eqref{eq:qr-def} and the choice of $\vert q_i \vert = 1$.
From here,  \eqref{eq:lift_condition} clearly resolves to
\begin{align*}
D\Upsilon_{(P, \eta_i)}(\id) \left[ (U, W_i) \right] & = (PU, 0) \\
& = f((P, \eta_i), U),
\end{align*}
as required.
This completes that proof that $\lambda$ is a velocity lift.
\end{proof}

The kinematics of the true state $\xi = (P, \eta_i) \in \mr{\totT}_{n_p,n_j}(3)$ of the VSLAM system are given by
\begin{align} \label{eq:true_system_state}
\dot{\xi} &= f(\xi, U).
\end{align}
Choose a reference configuration $\mr \xi = (\mr P, \mr{\eta}_i) \in \mr{\totT}_{n_p,n_j}(3)$.
By construction, the trajectories of the lifted system kinematics
\begin{align*}
\dot{X} = X \lambda(\Upsilon(X, \mr{\xi}), U)
\end{align*}
project to trajectories of the VSLAM kinematics \eqref{eq:true_system_state} via $\xi(t) = \Upsilon(X(t), \mr{\xi})$.

\section{Observer Design} \label{sec:observer}

\subsection{Observer Kinematics}
%
%
Define the observer state to lie on the VSLAM group, $\hat{X} = (\hat{A}, \hat{Q}_i) \in \VSLAM_n(3)$, with kinematics given by
\begin{align} \label{eq:group_observer_state}
\frac{\td}{\td t} \hat{X} &= \hat{X} \lambda(\Upsilon(\hat{X}, \mr \xi), U) + \hat{X} \Delta_{\hat{X}}, \notag \\
\hat{X}(0) &= \id,
\end{align}
where $\Delta_{\hat{X}} = (\Delta_{\hat{A}}, \Delta_{\hat{Q}_i}) \in \vslam_n(3)$ is an innovation term.
The estimated state $\hat{\xi} = (\hat{P}, \hat{\eta}_i) \in \mr{\totT}_{n_p,n_j}(3)$ is given by
\begin{align} \label{eq:estimated_system_state}
\hat{\xi} = \Upsilon(\hat{X}, \mr \xi).
\end{align}
Additional notation is helpful in simplifying the expressions that follow in the observer design. Define
\begin{align} \label{eq:y-measurement-def}
\hat{y}_i := h((\hat{P}, \hat{\eta}_i)),
\hspace{1cm} y_i := h(({P}, {\eta}_i)).
\end{align}
All expressions above are well-defined for equivalence classes in the SLAM manifold.

\subsection{Landmark Observer}
\begin{theorem} \label{th:observer}
Let $\xi = (P, \eta_i) \in \mr{\totT}_{n_p,n_j}(3)$ be the true state of the system, evolving with the kinematics \eqref{eq:true_system_state}. Let $\mr{\xi} \in \mr{\totT}_{n_p,n_j}(3)$ be arbitrary up to the requirement that, for all $i$, $\mr{\eta}_i$ and $\eta_i$ are members of the same orbit of $\RP^3$ under the action of $\SOT(3)$. Define $\hat{X} = (\hat{A}, \hat{Q}_i) \in \VSLAM_n(3)$ to be the observer state with kinematics defined by \eqref{eq:group_observer_state}, and define $\hat{\xi} = (\hat{P}, \hat{\eta}_i)$ as in \eqref{eq:estimated_system_state}.

Now, for $i = 1,..., n_p$, define $\Sigma_i \in \R^{3\times 3}$ by
\begin{align} \label{eq:Riccati Dynamics}
\dot{\Sigma}_i &= \Sigma_i \Omega_U^\times - \Omega_U^\times \Sigma_i + H_i - \Sigma_i \Pi_{y_i} G_i \Pi_{y_i} \Sigma_i, \notag \\
\Sigma(0)_i &= \Sigma_{i,0} > 0, \ G_i = k_G I_3, \ H_i = k_H I_3,
\end{align}
where $k_G,k_H > 0$ are constants, and assume that there exist $\delta > 0$ and $\mu > 0$ such that
\begin{align} \label{eq:riccati-observability-controllability}
\frac{1}{\delta} \int_t^{t+\delta} \Pi_{R_P(s)y_i(s)} ds &\geq \mu I_3.
\end{align}
for any time $t >0$ and for any $i=1,...,n_p$.
For $i = n_p+1,..., n_p+n_b$, define
\begin{align} \label{eq:bearing-landmark-no-riccati}
\Sigma_i \equiv I_3, \ G_i &= I_3, \ H_i = I_3.
\end{align}
Then, for every landmark $i=1,...,n_p+n_b$, define $\Delta_{\hat{Q}_i}$ as
\begin{align} \label{eq:landmark-innovation}
\Delta_{\hat{Q}_i} &= \left( \begin{matrix}
\left( \hat{y}_i^\times K_i \Pi_{y_i} \hat{y}_i \right)^\times & 0 \\ 0 & - \hat{y}_i^\top K_i \Pi_{y_i} \hat{y}_i
\end{matrix} \right), \notag \\
K_i &= k \Sigma_i \Pi_{y_i} G_i, \hspace{1cm} k > 0.5,
\end{align}
where $y_i$ and $\hat{y}_i$ are given by \eqref{eq:y-measurement-def}.
Let the innovation term $\Delta_{\hat{A}}$ be given by the least-squares solution to
\begin{align} \label{eq:robot-innovation-cost}
&\min_{(\Delta_{\hat{R}}, \Delta_{\hat{x}})} \sum_{i=1}^n \left\vert \frac{1}{\vert \hat{\theta}_i \vert} \ob{\Pi}_{\hat{\theta}_i} \left( \left( \begin{matrix}
-(\hat{\theta}_i^{1:3})^\times & \hat{\theta}_i^4 I_3 \\ 0 & 0
\end{matrix} \right) \left( \begin{matrix}
\Delta_{\hat{R}} \\ \Delta_{\hat{x}}
\end{matrix} \right) + \Delta_{\hat{Q}_i} \hat{\theta}_i \right) \right\vert, \notag \\
&\hspace{1.5cm}\Delta_{\hat{A}} = \left( \begin{matrix}
\Delta_{\hat{R}}^\times & \Delta_{\hat{x}} \\ 0 & 0
\end{matrix} \right),
\hspace{1cm} \hat{\theta}_i := \hat{P}^{-1} \hat{\eta}_i.
\end{align}
Then the estimated state coordinates $\hat{\xi}$ converge to the true coordinates $\xi$ almost-globally asymptotically and exponentially in the large\footnote{For any compact set in the basin of attraction of the equilibrium, the value of the Lyapunov function converges exponentially to zero.} up to equivalence on the SLAM manifold $\calM_n(3)$.
\end{theorem}

\begin{proof}
To verify that $\Delta_{\hat{A}}$ is well-defined note that the cost in \eqref{eq:robot-innovation-cost} is invariant to scale in the data $\hat{\theta}_i \mapsto a_i \hat{\theta}_i$ for $a_i \in \R \setminus \{ 0 \}$.
A Lyapunov analysis proves the desired result.

For $i=1,...,n_p$, recalling \eqref{eq:inverse-rp3-map-gamma}, define the error coordinates and candidate storage function as
\begin{align*}
e_i &:= \gamma(\hat{P}^{-1} \hat{\eta}_i) - \gamma(P^{-1} \eta_i), \\
l_i &:= \frac{1}{2} e_i^\top \Sigma_i^{-1} e_i,
\end{align*}
respectively.
The condition \eqref{eq:riccati-observability-controllability} ensures that $\Sigma_i$ is well-conditioned, and remains bounded and positive-definite for all time $t \geq 0$ \cite{2016_Hamel_cdc}.
Therefore the candidate storage function $l_i$ is positive definite.
It remains to show that $l_i$ is monotonically decreasing.
The kinematics of $e_i$ are
\begin{align*}
\dot{e}_i &= -\Omega_U^\times e_i - K_i \Pi_{y_i} e_i.
\end{align*}
Differentiating the candidate storage function, one has
\begin{align*}
\dot{l}_i &= e_i^\top \Sigma_i^{-1} \dot{e}_i - \frac{1}{2} e_i^\top \Sigma_i^{-1} \dot{\Sigma}_i \Sigma_i^{-1} e_i, \\
&= e_i^\top \Sigma_i^{-1} (-\Omega_U^\times e_i - k \Sigma_i \Pi_{y_i} G_i \Pi_{y_i} e_i) - \frac{1}{2} e_i^\top \Sigma_i^{-1} (\Sigma_i\Omega_U^\times \\
& \hspace{0.5cm} - \Omega_U^\times \Sigma_i + H_i - \Sigma_i \Pi_{y_i} G \Pi_{y_i} \Sigma_i) \Sigma_i^{-1} e_i, \\
&= -\frac{1}{2} e_i^\top \Sigma_i^{-1} \Omega_U^\times e_i -\frac{1}{2} e_i^\top \Omega_U^\times \Sigma_i^{-1} e_i \\
& \hspace{0.5cm} + \left( \frac{1}{2} - k \right) e_i^\top \Pi_{y_i} G \Pi_{y_i} e_i - \frac{1}{2} e_i^\top \Sigma_i^{-1} H \Sigma_i^{-1} e_i, \\
&\leq - \frac{1}{2} e_i^\top \Sigma_i^{-1} H_i \Sigma_i^{-1} e_i, \\
&\leq - \frac{1}{2} \frac{\sigma_{i,m}^2}{\sigma_{i,M}} k_H l_i,
\end{align*}
where $\sigma_{m,i}$ and $\sigma_{M,i}$ denote the infinum of the smallest and the supremum of the largest eigenvalues of $\Sigma_i$ over time, respectively.
Since $k_H > 0$ is chosen as a constant, and $\Sigma_i$ remains well-conditioned and bounded, the equilibrium $e_i = 0$ is exponentially stable.
Equivalently, this provides that $\hat{P}^{-1} \hat{\eta}_i \to P^{-1} \eta_i$ globally exponentially.

For $i=n_p+1,...,n_p+n_b$, define the candidate storage function
\begin{align} \label{eq:lyapunov-bearing-landmark}
l_i &:= \frac{1}{2} \left( 1 - \left( \frac{y_i^\top \hat{y}_i}{|y_i||\hat{y}_i|} \right)^2 \right).
\end{align}
Observe that $l_i$ is well-defined as a function of $\RP^2$ elements, since the expression is invariant to multiplication of $y_i$ or $\hat{y}_i$ by any non-zero scalar.
Clearly $l_i$ is positive definite.
The kinematics of the bearing $y_i \in \RP^2$ are given by
\begin{align*}
\dot{y}_i &= \frac{\td}{\td t} \begin{pmatrix} I_3 & \bf{0} \end{pmatrix} P^{-1} \eta_i, \\
&= - \begin{pmatrix} I_3 & \bf{0} \end{pmatrix} \ob{\Pi}_{P^{-1} \eta_i} W_i P^{-1} \eta_i, \\
&= - \Pi_{y_i} \Omega_U^\times \begin{pmatrix} I_3 & \bf{0} \end{pmatrix} P^{-1} \eta_i, \\
&= - \Omega_U^\times y_i.
\end{align*}
This is well-defined as an element of the tangent space $T_{y_i}\RP^2$ since any scaling of $y_i$ results in the same scaling of the expression for $\dot{y}_i$.
Since $\dot{y}_i^\top y_i = 0$, the dynamics of the norm of any chosen representative of $y_i$ are given by $\frac{\td}{\td t} \vert y_i \vert = 0$.
Analogously, recalling \eqref{eq:bearing-landmark-no-riccati} and \eqref{eq:landmark-innovation}, the kinematics of $\hat{y}_i \in \RP^2$ are given by
\begin{align*}
\dot{\hat{y}}_i &= (-\Omega_U^\times - \left( \hat{y}_i^\times K_i \Pi_{y_i} \hat{y}_i \right)^\times) \hat{y}_i, \\
&= -\Omega_U^\times \hat{y}_i + \hat{y}_i^\times \hat{y}_i^\times (k \Sigma_i \Pi_{y_i} G_i) \Pi_{y_i} \hat{y}_i, \\
&= -\Omega_U^\times \hat{y}_i - k \Pi_{\hat{y}_i} \Pi_{y_i} \hat{y}_i,
\end{align*}
and hence the dynamics of the norm of any representative of $\hat{y}_i$ are given by $\frac{\td}{\td t} \vert \hat{y}_i \vert = 0$.
As a consequence of this and the scale invariance of \eqref{eq:lyapunov-bearing-landmark}, we may choose $|y_i| = |y_i| = 1$ for readability without loss of generality.
Differentiating the candidate storage function leads to
\begin{align*}
\dot{l}_i &= - (y_i^\top \hat{y}_i)(\dot{y}_i^\top \hat{y}_i + y_i^\top \dot{\hat{y}}_i), \\
&= k (y_i^\top \hat{y}_i) y_i^\top \Pi_{\hat{y}_i} \Pi_{y_i} \hat{y}_i , \\
&= k (y_i^\top \hat{y}_i)^2 ( (y_i^\top \hat{y}_i)^2 - 1) \\
&= -k (y_i^\top \hat{y}_i)^2 l_i
\end{align*}
which is negative definite as long as the initial directions $y_i(0)$ and $\hat{y}_i(0)$ are not orthogonal. There are two situations in which $\dot{l}_i = 0$. The first one corresponds to the stable case where $l_i = 0$ ($\hat{y}_i$ and $y_i$ are parallel) while the second one corresponds to the unstable case for which $l_i = 1$ ($\hat{y}_i$ and $y_i$ are orthogonal).
To prove the exponential stability in the large, suppose that $0 < l_i \leq \epsilon < 1$ for some fixed $\epsilon$. Then,
\begin{align*}
\dot{l}_i &= - k (y_i^\top \hat{y}_i)^2 l_i, \\
&= -k (1-l_i)l_i, \\
&\leq -k (1-\epsilon)l_i.
\end{align*}
Observe that, unless $l_i = 1$, such an $\epsilon$ can always be found.
Therefore, $l_i \to 0$  almost-globally asymptotically, and exponentially in the large.
Since the measurement function $h$ is invertible on bearing-type elements, this provides the desired result that $\hat{P}^{-1} \hat{\eta}_i \to P^{-1} \eta_i$ almost-globally asymptotically and exponentially in the large.

Define the whole-of-system Lyapunov function
\begin{align} \label{eq:lyapunov}
\Lyap := \sum_{i=1}^n l_i.
\end{align}
From the analysis of each individual $l_i$, it is clear that $\Lyap \to 0$ almost-globally asymptotically and exponentially in the large.
The convergence of each $\Lyap$ provides that
\begin{align*}
(\hat{P}, \hat{\eta}_i) &\simeq ((P \hat{P}^{-1}) \hat{P}, (P \hat{P}^{-1}) \hat{\eta}_i), \\
&= (P, P (\hat{P}^{-1} \hat{\eta}_i)), \\
& \to (P, P (P^{-1} \eta_i)), \\
&= (P, \eta_i),
\end{align*}
almost-globally asymptotically and exponentially in the large as well. This completes the proof.
\end{proof}

\section{Simulation Results} \label{sec:simulation}
To verify the observer derived in Theorem \ref{th:observer}, we conducted a simulation of a vehicle equipped with a single monocular camera, observing 4 point-type landmarks and 2 bearing-type landmarks as it moves through space.
The vehicle moves in a circular trajectory at a fixed height of 3 m.
The body-fixed velocity $U$ is fixed to be constant, with $\Omega_U = (0, 0, -0.5)^\top$ rad/s and $V_U = (1.5, 0, 0)$ m/s.
For simplicity, the camera frame is assumed to coincide with the body-fixed frame of the vehicle, which avoids the need for a separate computation to transform the body-fixed velocity into the camera frame.
Let the true state be $(P, \eta_i) \in \mr{\totT}_{n_p, n_b}(3)$.
The reference configuration is chosen as $\mr \xi = (I_4, \mr{\eta}_i)$, where
\begin{align*}
\mr{\eta}_i = \alpha \left( 2 \left( \frac{h(\eta_i)}{\vert h(\eta_i) \vert} + \epsilon_i \right) \right)
\end{align*}
where the $\epsilon_i$ terms represent errors in the initial measurements.
The observer is defined on $\VSLAM_n(3)$, with kinematics given by \eqref{eq:group_observer_state} and innovation terms given by Theorem \ref{th:observer}.
The initial conditions and gains for the observer are chosen as
\begin{align*}
\Sigma_i(0) = 25I_3, \ k_H = 0.5, \ k_G = 2.0, \ k = 1.0.
\end{align*}
The simulation was carried out by implementing the continuous time system with Euler integration using a time step of $dt = 0.02$ s.

Figure \ref{fig:lyapunov} shows the evolution of $\log_{10}(\Lyap)$, where $\Lyap$ is the Lyapunov function of the simulated system as defined in \eqref{eq:lyapunov}.
This clearly shows exponential convergence of the observer error dynamics.
Figure \ref{fig:trajectory} shows the evolution of the trajectory of the simulated system.
Since the estimated state only converges to the true state up to equivalence on the SLAM manifold $\calM_n(3)$, it is necessary to assign total space coordinates to the estimate to aid the comparison.
In Figure \ref{fig:trajectory} the choice of total space coordinates for the estimated state is made so that the final robot pose is aligned with that of the true state.
This shows that the landmarks have correctly converged to the true landmarks up to the SLAM manifold equivalence.

\begin{figure}[!htb] \centering
\subfloat[The evolution of $\log_{10}$ of the Lyapunov function $\Lyap$ \eqref{eq:lyapunov} with respect to time.]{ \centering
\includegraphics[width=0.4\linewidth]{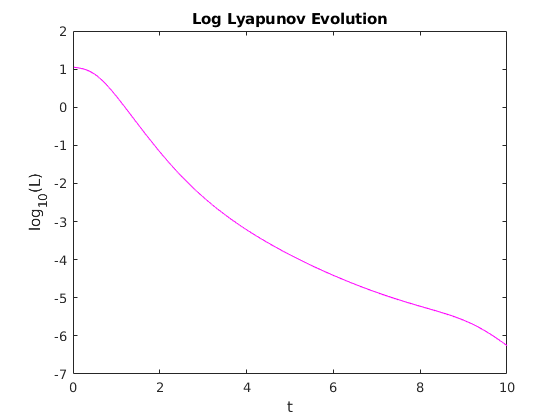}
\label{fig:lyapunov}}\hspace{0.04\linewidth}%
\subfloat[The trajectory of the simulated system (green, blue) compared with the true system evolution (black, red).]{ \centering
\includegraphics[width=0.4\linewidth]{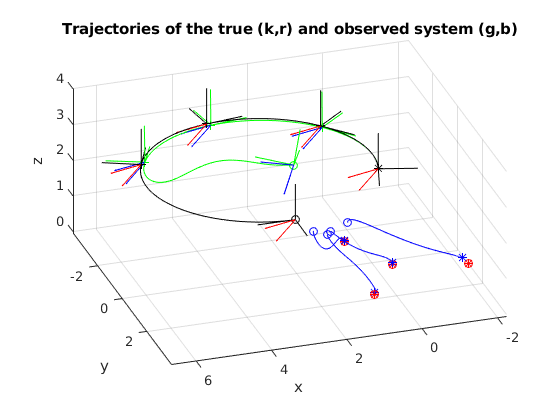}
\label{fig:trajectory}
}
\end{figure}


\section{Conclusion}
This paper presents an observer design posed on a novel symmetry group for the visual SLAM problem.
The total space and SLAM manifold conceptualised in \cite{2017_Mahony_cdc} have been extended to include free vectors.
The development of the symmetry group $\VSLAM_n(3)$ has allowed both point-type and bearing-type landmarks to be treated in a unified framework.
Riccati observers were incorporated for each of the point-type landmarks, and grant the user refined control over their convergence.
The almost-global convergence of the proposed observer on both point-type and bearing-type landmarks is a contrast to many state-of-the-art Extended Kalman Filter systems, which suffer from linearisation errors.
While research into the development of non-linear observers for the SLAM problem is only recent, the observer for visual SLAM presented in this paper demonstrates some of the key advantages the approach can offer.

\section*{Acknowledgment}

This research was supported by the Australian Research Council
through the ``Australian Centre of Excellence for Robotic Vision'' CE140100016.

\bibliographystyle{plain}
\bibliography{references}

\end{document}